\documentclass{article}
\usepackage{ijcai22}

\usepackage{times}
\usepackage{soul}
\usepackage{url}
\usepackage[hidelinks]{hyperref}
\usepackage[utf8]{inputenc}
\usepackage[small]{caption}
\usepackage{graphicx}
\usepackage{amsmath}
\usepackage{amsthm}
\usepackage{booktabs}
\usepackage{algorithm}
\usepackage{algorithmic}
\usepackage{dsfont}
\usepackage{xr}
\usepackage{amssymb}

\makeatletter
\newcommand{\rmnum}[1]{\romannumeral #1}
\newcommand{\Rmnum}[1]{\expandafter\@slowromancap\romannumeral #1@}
\makeatother
\newtheorem{theorem}{Theorem}
\newtheorem{lemma}{Lemma}
\newtheorem{remark}{Remark}

\title{The Implicit Regularization of Momentum Gradient Descent with Early Stopping}

\author{
Li Wang$^1$ \and
Yingcong Zhou$^2$\And
Zhiguo Fu$^{1*}$
\affiliations
$^1$Northeast Normal University\\
$^2$ Beihua University
\emails
wangl024@nenu.edu.cn,
zycong0821@163.com,
fuzg432@nenu.edu.cn
}

\begin{document}

\maketitle

\begin{abstract}
The study on the implicit regularization induced by gradient-based optimization is a longstanding pursuit.
In the present paper, we characterize the implicit regularization of momentum gradient descent (MGD) with early stopping by comparing with the explicit $\ell_2$-regularization (ridge). In details, we study MGD in the continuous-time view, so-called momentum gradient flow (MGF), and show that its tendency is closer to ridge than the gradient descent (GD)~\cite{Ali2019} for least squares regression. Moreover, we prove that, under the calibration $t=\sqrt{2/\lambda}$, where $t$ is the time parameter in MGF and $\lambda$ is the tuning parameter in ridge regression, the risk of MGF is no more than 1.54 times that of ridge. In particular, the relative Bayes risk of MGF to ridge is  between 1 and 1.035 under the optimal tuning. The numerical experiments support our theoretical results strongly.
\end{abstract}

\section{Introduction}
Implicit regularization refers to the optimization algorithm's preference to implicitly choosing certain structured solution as if
some explicit regularization term appeared in its objective.
The study on the implicit regularization of optimization can be dated back to at least 30 years
~\cite{1990Generalization}.
Recently, it has been shown that the implicit regularization of optimization may be a key to understanding the generalization mystery of deep learning~\cite{Zhang2016UnderstandingDeepLearning}.
\nocite{Liu2020BadGM}
\nocite{Tarmoun2021UnderstandingTD}
\nocite{Zhang2021UnderstandingDL}
After that,
a series of studies on
the implicit regularization of optimization
for the various settings were launched,
including matrix factorization \cite{Gunasekar2018ImplicitRI,Arora2019ImplicitRI},
\nocite{Li2021TowardsRT}
classification tasks
~\cite{Soudry2018TheIB,Lyu2020GradientDM}
\nocite{Gunasekar2018CharacterizingIB}
and nonlinear neural networks 
~\cite{Vardi2021ImplicitRI}, etc.

In this paper, we aim to characterize the implicit regularization of MGD~\cite{POLYAK19641}, which is one of the most popular optimization algorithms in practice because of its ability to accelerate learning, especially for the cases of high curvature, small but consistent gradients, or noisy gradients. Many other variants and improvements of MGD have been developed in~\cite{Lin2020AcceleratedOF,Even2021ACV},
\nocite{Oviedo2021}
and their convergence behaviors have been analyzed. It has been empirically observed that MGD and its variants (eg, Nesterov) perform well in deep learning. However, there is a lack of theoretical discussions to uncover how MGD affects generalization performance, which is our consideration in the present paper.

An important way to explore the implicit regularization of optimization is to compare the optimization paths with the explicit regularization paths.
Recently, 
~\cite{Suggala2018,Ali2019}, in a continuous-time view, showed how the optimization path of GD is (point-wise) closely connected to an explicit $\ell_2$ regularization. In a similar idea,~\cite{Ali2020TheIR} studied the implicit regularization of stochastic gradient descent (SGD). Furthermore,~\cite{Zou2021TheBO} showed that the generalization performance of SGD is always no worse than that of ridge regression in a wide range of overparameterized problems. More results can be found in~\cite{Barrett2021ImplicitGR,Steinerberger2021OnTR}
\nocite{Smith2021StochasticGD}.
In the present paper, we study the implicit regularization of MGD by comparing its path to the path of ridge.
%
%
Note that MGD is a second-order iteration essentially, so the corresponding continuous-time form MGF is a second-order differential equation.
We find the analytical solution of MGF by the singular value
decomposition of the data matrix.
But the solution is involved and
it is more challenging to
analysis its asymptotic behavior.
The main contributions of this paper are as follows:
\begin{itemize}
\item We give the analytical continuous-time form of MGD, called MGF (a second order differential equation),
 and prove that MGD (point-wise) convergences to MGF as the step size $\epsilon\rightarrow 0$.
\item We find that MGF can be expressed as solutions to a sequence of $\ell_2$
regularized least squares problems, and then set the calibration of early stopping  $t=\sqrt{2/\lambda}$
by Taylor expansion.
\item We show that the risk of MGF at time $t$ is no more than
1.54 times that of ridge regression at tuning parameter $\lambda=2/t^2$.
And the ratio of the Bayes risk of  MGF to that of ridge
 is between 1 and 1.035 under the optimal tuning.
\item We analyse the limiting behaviors of the risk of MGF
in the Marchenko-Pastur asymptotic model, where $p/n$ (the ratio of the feature dimension to sample size)
converges to a positive constant.
\item We carry out the numerical experiments to verify the coupling between MGF and ridge.
The results show that they have the extremely similar paths,
which confirm the implicit regularization of MGD.
\end{itemize}

\section{The Continuous-time Forms of MGD and Ridge}\label{section2}

\subsection{Momentum Gradient Flow}\label{section21}
Let $X\in \mathbb{R}^{n\times p}$, a column full-rank matrix, is the data matrix and $y\in \mathbb{R}^n$ is the response vector.
We would like to analyse the learning by minimizing the loss functions of the following form,
\begin{equation}\label{1.1}
\min\limits_{\beta\in \mathbb{R}^p}~~L(f(X;\beta),y),
\end{equation}
where $L$ is the loss function, $\beta\in \mathbb{R}^p$ is the weight vector and $f(X;\beta)$ is the predicted output when
the input is $X$. We are particularly interested in the implicit regularization of
MGD applied to (\ref{1.1}).
In the standard form, we have the MGD iterations
\begin{eqnarray}
&&\tilde{v}_{k+1}=\tilde{\mu}\tilde{v}_k-\tilde{\epsilon}g(\beta_k), \nonumber\\
&&\beta_{k+1}=\beta_k+\tilde{v}_{k+1},\nonumber
\end{eqnarray}
where $g(\beta_k)=\nabla_\beta L(f(X;\beta),y)$; $\tilde{\epsilon}>0$ is the step size;
$v$ is momentum  which is set to an exponentially decaying average of the negative gradient;
and $\tilde{\mu} \in (0,1)$ (typically close to 1) is the momentum parameter that determines how
quickly the contributions of previous gradients exponentially decay.
To facilitate the following analysis, we consider a rescaled version of the general MGD.
By  redefining
$\epsilon=\sqrt{\tilde{\epsilon}}, v_k=\frac{\tilde{v}_k}{\sqrt{\tilde{\epsilon}}}$ and $\mu=\frac{1-\tilde{\mu}}{\sqrt{\tilde{\epsilon}}}$,
we have 
\begin{align}
\begin{split}\label{1.2}
&v_{k+1}=v_k-\mu\epsilon v_k-\epsilon g(\beta_k), \\
&\beta_{k+1}=\beta_k+\epsilon v_{k+1}.
\end{split}
\end{align}
After rescaling, we have the momentum parameter $\mu \in (0,\epsilon^{-1/2})$ from $\tilde{\mu}\in(0,1)$. It follows that
\begin{equation}
\beta_{k+1}=\beta_k+\epsilon v_k-\mu\epsilon^2v_k-\epsilon^2g(\beta_k). \nonumber
\end{equation}
Moreover, let $v_k=\frac{\beta_{k}-\beta_{k-1}}{\epsilon}$, then we have
\begin{equation}\label{1.3}
\beta_{k+1}=2\beta_{k}-\beta_{k-1}-\mu\epsilon(\beta_{k}-\beta_{k-1})-\epsilon^2g(\beta_k).
\end{equation}
(\ref{1.3}) shows that MGD is a second-order iteration essentially.
And note that the initial values  $v_0$ and $\beta_0$  in (\ref{1.2}) can produce $\beta_0$ and $\beta_1$ for (\ref{1.3}) to iterate, and vice versa.
Thus (\ref{1.2}) and (\ref{1.3}) are equivalent.
Rearranging (\ref{1.3}) yields that
\begin{equation}
\frac{\beta_{k+1}+\beta_{k-1}-2\beta_{k}}{\epsilon^2}+\mu\frac{(\beta_{k}-\beta_{k-1})}{\epsilon}=-g(\beta_k). \nonumber
\end{equation}
Letting $\epsilon\rightarrow0$, we get the continuous-time form of MGD
\begin{equation}\label{1.4}
\beta^{\prime\prime}(t)+\mu\beta^\prime(t)=-g(\beta(t)),
\end{equation}
over time $t \geq 0$. We call (\ref{1.4}) the momentum gradient flow (MGF) which are the second-order differential equations.
In this paper, we focus on the analysis of implicit regularization of MGF through the least squares problem, which is
\begin{equation}\label{1.5}
\min\limits_{\beta\in R^p}~~L(f(X;\beta),y)=\frac{1}{2n}||y-X\beta||_2^2,
\end{equation}
and $g(\beta_k)=\frac{1}{n}X^{\rm T}X\beta_k-\frac{1}{n}X^{\rm T}y$.
To facilitate the following analysis, it is helpful to consider the singular value
decomposition of $X$. Let $X=\sqrt{n}US^{\frac{1}{2}}V^{\rm T}$ be the singular value decomposition, thus $X^{\rm T}X/n=VSV^{\rm T}$
is the eigendecomposition, where $S={\rm diag}(s_i)(i=1,\cdots,p)$ and $s_i$ are the eigenvalues of $X^{\rm T}X/n$ satisfying $s_1\geq s_2\geq \cdots s_p>0 $. We note that $X^{\rm T}X/n$ is a symmetric positive definite matrix, since $X$ has the rank $p$. 
And then applying MGD to (\ref{1.5}) initialized at $v_0=-\frac{\epsilon X^{\rm T}y}{2n(1-\mu \epsilon)}$ and $\beta_0=0$ (which implies that
$\beta_1=\frac{1}{2n}\epsilon^2 X^{\rm T}y$ by (\ref{1.2})),
we have the iterations
\begin{align}\label{1.6}
\beta_{k+1}=&2\beta_{k}-\beta_{k-1}-\epsilon D(\mu)(\beta_{k}-\beta_{k-1})    \nonumber\\
&-\epsilon^2\left(VSV^{\rm T}\beta_k-\frac{1}{\sqrt{n}}VS^{\frac{1}{2}}U^{\rm T}y\right),
\end{align}
where $D(\mu)={\rm diag}(\mu)$
and the corresponding MGF is
\begin{align}\label{1.7}
\beta^{\prime\prime}(t)+D(\mu)\beta^\prime(t)+VSV^{\rm T}\beta(t)=\frac{1}{\sqrt{n}}VS^{\frac{1}{2}}U^{\rm T}y
\end{align}
for $t\geq0$, which subjects to the initial conditions $\beta(0)=0,~\beta'(0)=0$.
%
%
%
%
Now, we derive the exact solution of MGF.
\begin{lemma}\label{lemma1}
Fix a response $y$ and a data matrix $X$. The MGF (\ref{1.7}), subject to $\beta(0)=0, \beta'(0)=0$ and $D(\mu)\succ2S^{1/2}$
admits the exact solution
\begin{equation}\label{1.8}
\hat{\beta}^{\rm mgf}(t)=\frac{1}{\sqrt{n}}VS^{-1}\left(I-H(S,t)\right)S^{\frac{1}{2}}U^{\rm T}y
\end{equation}
where
\begin{align}
H(S,t)=&\left(2\sqrt{D(\mu)^2-4S}\right)^{-1}
\left[\left(D(\mu)+\sqrt{D(\mu)^2-4S}\right)\cdot\right.\nonumber\\
&\left.\exp{\left(\frac{1}{2}\left(-D(\mu)
+\sqrt{D(\mu)^2-4S}\right)t\right)}\right.       \nonumber\\
&\left.+\left(D(\mu)-\sqrt{D(\mu)^2-4S}\right)\cdot
\right.       \nonumber\\
&\left.\exp{\left(\frac{1}{2}\left(-D(\mu)-\sqrt{D(\mu)^2-4S}\right)t\right)}\right].\nonumber
\end{align}
\begin{proof}
The result follows from solving the second-order differential equations (\ref{1.7})-\emph{see Supplement}. 
\end{proof}
\end{lemma}
Throughout this paper, $H(S,t)$ is defined as above, $\prec$ denotes the Loewner ordering on the matrices
(i.e., $A\prec B$ means that $B-A$ is positive definite),
 $\|v\|$ denotes the Euclidean norm of a vector $v$ and  $\|A\|$ denotes the spectral norm of a matrix $A$.
The following lemma shows that MGD (point-wise) convergences to MGF as the step size $\epsilon\rightarrow 0$.
\begin{lemma}\label{lemma2}
For least squares (\ref{1.5}), consider discrete-time MGD $\{\beta_k:k=0,\cdots,n\}$ (\ref{1.6}) initialized at $v_0=-\frac{\epsilon X^{\rm T}y}{2n(1-\mu \epsilon)}$ and $\beta_0=0$, MGF $\left\{\beta(t):t\in[0,T]\right\}$ (\ref{1.8}) subjects to $\beta(0)=0, \beta^\prime(0)=0$. For $n+1=\lfloor T/\epsilon\rfloor$, it holds that
\begin{equation}
\|\hat{\beta}^{\rm mgf}(t_{k+1})-\beta_{k+1}\|\leq \epsilon CT^2\exp(s_{\rm max}T^2),
\end{equation}
where $t_{k+1}=(k+1)\epsilon$, $C$ is a positive constant and $s_{\rm max}$ is the largest eigenvalue of $X^{\rm T}X/n$.
\end{lemma}
\begin{proof}
The uniform bound is given by  numerical analysis-\emph{see Supplement}.
\end{proof}
By Lemma~\ref{lemma2}, we will focus on the exact solution of the continuous-time MGF to
study the implicit regularization of MGD in the following.

\subsection{Basic Comparisons Between MGF and Ridge}\label{section22}
Consider the ridge regression, the $\ell_2$ regularized version of (\ref{1.5}), that is
\begin{equation}\label{2.1}
\min\limits_{\beta\in R^p}~~\frac{1}{2n}||y-X\beta||_2^2+\lambda||\beta||_2^2,
\end{equation}
where $\lambda >0$ is a tuning parameter. The explicit ridge solution is
\begin{equation}\label{2.2}
\hat{\beta}^{\rm ridge}(\lambda)=(X^{\rm T}X+n\lambda I)^{-1}X^{\rm T}y.
\end{equation}
To compare the paths of ridge (\ref{2.2}) and MGF (\ref{1.8}), it is helpful to rearrange them, on the scale of fitted values, to
\begin{align}\label{2.3}
X\hat{\beta}^{\rm ridge}(\lambda)&=US^{\frac{1}{2}}V^{\rm T}V(S+\lambda I)^{-1}S^{\frac{1}{2}}U^{\rm T}y  \nonumber\\
&= US(S+\lambda I)^{-1}U^{\rm T}y.\\
\label{2.4}
X\hat{\beta}^{\rm mgf}(t)&=
US^{\frac{1}{2}}V^{\rm T}VS^{-1}\left(I-H(S,t)\right)S^{\frac{1}{2}}U^{\rm T}y   \nonumber\\
&=U\left(I-H(S,t)\right)U^{\rm T}y.
\end{align}
Letting $u_i\in \mathbb{R}^n, i=1,\cdots, p$ denote the columns of $U$,
we see that (\ref{2.3}), (\ref{2.4}) are both linear smoothers (linear functions of $y$) of the form
$\sum_{i=1}^{p}\varphi(s_i,\kappa)\cdot u_iu_i^{\rm T}y$,
for a spectral shrinkage map $\varphi(\cdot,\kappa):[0,\infty)\rightarrow[0,\infty)$ and parameter $\kappa$.
This map is $\varphi^{\rm ridge}(s,\lambda)=s/(s+\lambda)$ for ridge, and $\varphi^{\rm mgf}(s,t)=1-H(s,t)$ for MGF.
We see that both apply more shrinkage for smaller values of $s$, i.e., lower-variance directions of $X^{\rm T}X/n$,
but do so in apparently different ways.
And the two shrinkage maps agree at the extreme ends (i.e., set $\lambda\rightarrow0$ and $t\rightarrow\infty$, $\varphi(s,\cdot)\rightarrow 1$, or $\lambda\rightarrow\infty$ and $t\rightarrow0$, $\varphi(s,\cdot)\rightarrow0$).
We note that the parametrization $\lambda=2/t^2$ (the calibration setting is obtained by
Taylor expansion and will be explained in Section \ref{section23})
gives the two shrinkage maps similar behaviors: see
Figure \ref{figure1} for a visualization. Moreover, as we will show later in Sections \ref{section4}-\ref{section6}, the two shrinkage maps
(under the calibration $\lambda=2/t^2$) lead to similar risk curves for MGF and ridge.
\begin{figure}[!h]
\centering
\includegraphics[width=9cm, height=4cm]{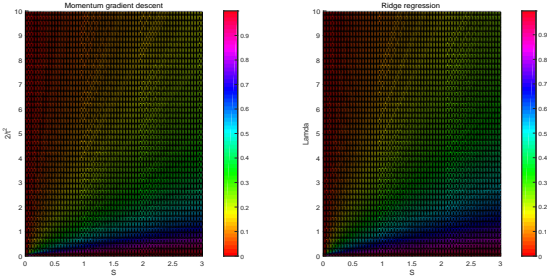}
\caption{\footnotesize{Comparison of MGF and ridge spectral shrinkage maps.}} \label{figure1}
\end{figure}

\subsection{Underlying Regularization Problems}\label{section23}
 We are interested in the connection between MGF and ridge. It is natural to wonder whether MGF can be expressed as solutions to sequences of regularized least squares. The following lemma confirms this fact.
\begin{lemma}\label{lemma3}
Fix $y$ and $X$. Under the initial conditions $\beta(0)=0, \beta'(0)=0$, for $t\geq0$, MGF (\ref{1.8}) uniquely solves the optimization problem
\begin{equation}\label{2.6}
\min\limits_{\beta\in R^p}~~\frac{1}{2n}\left\|y-X\beta\right\|_2^2+\beta^{\rm T}Q_t\beta,
\end{equation}
where $Q_t=VS\left(H(S,t)^{-1}-I\right)^{-1}V^{\rm T}$.
\end{lemma}
\begin{proof}
The result readily follows from MGF (\ref{1.8}) and the solution of the optimization problem-\emph{see Supplement}.
\end{proof}
\begin{remark}\label{remark1}
Computing the first two orders  of the Taylor’s Series of $H(S,t)^{-1}$ at the point $t=0$, we have
\begin{equation}
H(S,t)^{-1}=I+H'(S,t)^{-1}t+\frac{1}{2}H''(S,t)^{-1}t^2=I+\frac{1}{2}t^2S.    \nonumber
\end{equation}
An application of the claim of Lemma \ref{lemma3} can give the expression of regularization parameter
\begin{equation}
Q_t=VS(H(S,t)^{-1}-I)^{-1}V^{\rm T}=VS(\frac{1}{2}t^2S)^{-1}V^{\rm T}=\frac{2}{t^2}I.   \nonumber
\end{equation}
It shows that MGF is extremely close to ridge, under the calibration $\lambda=2/t^2$.
\end{remark}

\section{Measures of Risk}\label{section3}
\subsection{Estimation Risk}
For any feature matrix $X\in \mathbb{R}^{n\times p}$, 
we consider a generic response model,
\begin{equation}\label{3.1}
y\mid\beta_0\sim(X\beta_0, \sigma^2I),
\end{equation}
i.e., $E(y\mid\beta_0)=X\beta_0, {\rm Cov}(y\mid\beta_0)=\sigma^2I$ for the underlying coefficient vector $\beta_0\in \mathbb{R}^p$
and the error variance $\sigma^2>0$.
For an estimator $\hat{\beta}$ (i.e., measurable function of $X, y$), we define its estimation risk (or simply, risk) as
\begin{equation}\label{3.2}
{\rm {Risk}}(\hat{\beta}; \beta_0)=E[\| \hat{\beta}-\beta_0\|^2_2\mid\beta_0].
\end{equation}
We consider a spherical prior,
\begin{equation}\label{3.3}
\beta_0\sim(0,\frac{r^2}{p}I),
\end{equation}
for some signal strength $r^2=E\|\beta_0\|^2_2>0$, and define the Bayes risk of an estimator $\hat{\beta}$ as
\begin{equation}\label{3.4}
{\rm {Risk}}(\hat{\beta})=E\| \hat{\beta}-\beta_0\|^2_2.
\end{equation}
Next we give expressions for the risk and Bayes risk of MGF.
\begin{lemma}\label{lemma4}
Under the data model (\ref{3.1}), for $t\geq0$, the risk of the MGF (\ref{1.8}) is
\begin{align}\label{3.5}
{\rm Risk}&(\hat{\beta}^{\rm mgf}(t); \beta_0)=    \nonumber\\
&\sum_{i=1}^{p}\left[|\beta_0^{\rm T}v_i|^2H^2(s_i,t)+\frac{\sigma^2}{n}\frac{(1-H(s_i,t))^2}{s_i}\right],
\end{align}
and under the prior (\ref{3.3}), the Bayes risk is
\begin{align}\label{3.6}
{\rm Risk}&(\hat{\beta}^{\rm mgf}(t))=   \nonumber\\
&\frac{\sigma^2}{n}\sum_{i=1}^{p}\left[\alpha H^2(s_i,t)+\frac{(1-H(s_i,t))^2}{s_i}\right],
\end{align}
where $\alpha=r^2n/(\sigma^2p)$.
\end{lemma}
\begin{proof}
The results follow from the definitions of risk, Bayes risk and bias-variance decomposition-\emph{see Supplement}.
\end{proof}
\begin{remark}\label{remark2}
Compare (\ref{3.5}) to the risk of ridge regression,
\begin{align}\label{3.7}
{\rm Risk}&(\hat{\beta}^{\rm ridge}(\lambda); \beta_0)=    \nonumber\\
&\sum_{i=1}^{p}\left[|\beta_0^{\rm T}v_i|^2\frac{\lambda^2}{(s_i+\lambda)^2}+\frac{\sigma^2}{n}\frac{s_i}{(s_i+\lambda)^2}\right],
\end{align}
and compare (\ref{3.6}) to the Bayes risk of ridge,
\begin{equation}\label{3.8}
{\rm Risk}(\hat{\beta}^{\rm ridge}(\lambda))=\frac{\sigma^2}{n}\sum_{i=1}^{p}\left[\frac{\alpha\lambda^2+s_i}{(s_i+\lambda)^2}\right],
\end{equation}
where $\alpha=r^2n/(\sigma^2p)$. These ridge results follow from standard calculations, which can be found in many other papers; for
completeness, we give the details in Supplement.
\end{remark}

\subsection{Prediction Risk}
In this section, we analyse the prediction risk. Let
\begin{equation}\label{3.9}
x_0 \sim (0, \Sigma)
\end{equation}
for a positive semidefinite matrix $\Sigma\in \mathbb{R}^{p\times p}$, and assume that $x_0$ is independent of $y\mid\beta_0$. We define in-sample
prediction risk and out-of-sample prediction risk
as
\begin{align}
\label{3.10}&{\rm Risk^{in}(\hat{\beta};\beta_0)}=\frac{1}{n}E[\|X\hat{\beta}-X\beta_0\|^2_2\mid\beta_0],\\
\label{3.11}&{\rm Risk^{out}(\hat{\beta};\beta_0)}=E[(x_0^{\rm T}\hat{\beta}-x_0^{\rm T}\beta_0)^2\mid\beta_0],
\end{align}
respectively, and their Bayes versions as
$
{\rm Risk^{in}(\hat{\beta})}=\frac{1}{n}E[\|X\hat{\beta}-X\beta_0\|^2_2],
{\rm Risk^{out}(\hat{\beta};\beta_0)}=E[(x_0^{\rm T}\hat{\beta}-x_0^{\rm T}\beta_0)^2],
$
respectively.
Now, we give the expressions for the prediction risk and Bayes prediction risk of MGF.
\begin{lemma}\label{lemma5}
Under (\ref{3.1}) and (\ref{3.9}), the out-of-sample prediction risk of MGF (\ref{1.8}) is
\begin{align}\label{3.12}
 {\rm Risk^{out}}(\hat{\beta}^{\rm mgf}(t);&\beta_0)= \beta_0^{\rm T}VH(S,t)V^{\rm T}\Sigma VH(S,t)V^{\rm T}\beta_0 \nonumber\\
&+\frac{\sigma^2}{n}{\rm tr}\left[S^{-1}(I-H(S,t))^2\Sigma\right],
\end{align}
and under (\ref{3.3}), the Bayes out-of-sample prediction risk is
\begin{align}\label{3.13}
{\rm Risk^{out}}&(\hat{\beta}^{\rm mgf}(t))=   \nonumber\\
&\frac{\sigma^2}{n}{\rm tr}\left[\alpha H^2(S,t)\Sigma+S^{-1}(I-H(S,t))^2\Sigma\right].
\end{align}
\end{lemma}
\begin{proof}
The results follow from the definitions of the out-of-sample prediction risk, Bayes out-of-sample prediction risk and bias-variance decomposition-
\emph{see Supplement}.
\end{proof}
\begin{remark}\label{remark3}
Similar to (\ref{3.12}) and (\ref{3.13}), we have the out-of-sample prediction risk and  Bayes out-of-sample prediction risk of ridge
\begin{align}\label{3.14}
{\rm Risk^{out}}&(\hat{\beta}^{\rm ridge}(\lambda); \beta_0)=\nonumber\\
&\lambda^2\beta_0^{\rm T}V(S+\lambda I)^{-1}V^{\rm T}\Sigma V(S+
\lambda I)^{-1}V^{\rm T}\beta_0\nonumber\\
&+\frac{\sigma^2}{n}{\rm tr}[S(S+\lambda I)^{-2}\Sigma],\\
\label{3.15}
{\rm Risk^{out}}&(\hat{\beta}^{\rm ridge}(\lambda))=    \nonumber\\
&\frac{\sigma^2}{n}{\rm tr}\left[\lambda^2\alpha(S+\lambda I)^{-2}\Sigma+ S(S+\lambda I)^{-2}\Sigma\right],
\end{align}
respectively. More details can be found in Supplement. 
\end{remark}

\begin{remark}
The results of the in-sample prediction risk of MGF can be expressed as
\begin{align}\label{3.16}
&{\rm Risk^{in}}(\hat{\beta}^{\rm mgf}(t);\beta_0)=   \nonumber\\
&\sum_{i=1}^p\left[|\beta_0^{\rm T}v_i|^2s_iH^2(s_i,t)+\frac{\sigma^2}{n}(1-H(s_i,t))^2\right],
\end{align}
and the Bayse prediction in-sample risk can be expressed as
\begin{align}\label{3.17}
{\rm Risk^{in}}&(\hat{\beta}^{\rm mgf}(t))=    \nonumber\\
&\frac{\sigma^2}{n}\sum_{i=1}^p\left[\alpha s_iH^2(s_i,t)+(1-H(s_i,t))^2\right].
\end{align}
Similarly, we can give the ridge results,
\begin{align}\label{3.18}
 {\rm Risk^{in}}&(\hat{\beta}^{\rm ridge}(\lambda); \beta_0)= \nonumber\\
&\sum_{i=1}^{p}\left[|\beta_0^{\rm T}v_i|^2\frac{\lambda^2s_i}{(s_i+\lambda)^2}+\frac{\sigma^2}{n}\frac{s_i^2}{(s_i+\lambda)^2}\right],\\
\label{3.19}
{\rm Risk^{in}}&(\hat{\beta}^{\rm ridge}(\lambda))
=\frac{\sigma^2}{n}\sum_{i=1}^p\left[\frac{\alpha\lambda^2s_i+s_i^2}{(s_i+\lambda)^2}\right].
\end{align}
The proof can be found in the Supplement.
\end{remark}

\section{Relative Risk Bounds} \label{section4}
\subsection{Relative Estimation Risk and Prediction Risk}
In this section, we study the bound on the relative risk of MGF to ridge, under the calibration $\lambda=2/t^2$.
Firstly, we need to introduce two critical inequalities.
\begin{lemma}\label{lemma6}
For $t\geq0$ and $s_i>0$, we have
$\emph{(\rmnum{1})}~H(s_i,t)< 1.24\frac{1}{1+s_it^2/2}; \emph{(\rmnum{2})}~1-H(s_i,t)< 1.04\frac{s_it^2/2}{1+s_it^2/2}$.
\end{lemma}
\begin{proof} The results follow from the numerically computing-\emph{see Supplement}.
\end{proof}
The following theorem gives the bounds of the relative risk of MGF to ridge.
\begin{theorem}\label{theorem1}
Consider the data model (\ref{3.1}).  For all $\beta_0\in \mathbb{R}^p$ and $t\geq0$, we have
\begin{equation}\label{bound}
{\rm Risk}(\hat{\beta}^{\rm mgf}(t);\beta_0)<1.5376\cdot {\rm Risk}(\hat{\beta}^{\rm ridge}(2/t^2);\beta_0).
\end{equation}
Moreover, (\ref{bound}) also holds if we
replace  the risk by the Bayes risk  for any prior (\ref{3.3}), the in-sample prediction risk, the Bayes in-sample prediction risk for any prior (\ref{3.3}),
 or the Bayes out-of-sample prediction risk for
any prior (\ref{3.3}) and the feature distribution (\ref{3.9}).
%
\end{theorem}
\begin{proof}
For the risk, set $\lambda=2/t^2$ and denote
the $i$th summand of (\ref{3.5}) and (\ref{3.7}) by $a_i$ and $b_i$, respectively.
Then we have
\begin{align}
a_i&=|v_i^{\rm T}\beta_0|^2H^2(s_it)+\frac{\sigma^2}{n}\frac{[1-H(s_i,t)]^2}{s_i}  \nonumber\\
&<|v_i^{\rm T}\beta_0|^2 1.24^2\frac{1}{(1+s_it^2/2)^2}+\frac{\sigma^2}{n} 1.04^2\frac{s_i(t^2/2)^2}{(1+s_it^2/2)^2}\nonumber\\
&< 1.5376\cdot\left[|v_i^{\rm T}\beta_0|^2 \frac{(2/t^2)^2}{(s_i+2/t^2)^2}+\frac{\sigma^2}{n} \frac{s_i}{(s_i+2/t^2)^2}\right]\nonumber\\
&=1.5376\cdot b_i.  \nonumber
\end{align}
The first inequality follows from Lemma \ref{lemma6}. 
Then the bound of the risk follows by summing over $i = 1, . . . , p$.

We have the bound of the Bayes risk just by taking expectations on each side of (\ref{bound}).

For the in-sample prediction risk, we can get the bound by multiplying $s_i$ to each summand in (\ref{bound}).
By taking expectations for the in-sample prediction risk, we have
the bound of the Bayes in-sample prediction risk.

Since $H(S, t)$ and $S$ are diagonal matrices,
the two inequalities in Lemma \ref{lemma6} can be  extended to matrix operations, i.e.
$H^2(S,t)< 1.5376(I+St^2/2)^{-2}; \left(I-H(S,t)\right)^2< 1.0816(S^2t^4/4)(I+St^2/2)^{-2}$.
Note that $\Sigma\succeq0$ in (\ref{3.13}) and (\ref{3.15}).
Then for the  Bayes out-of-sample prediction risk,
we have
\begin{align}
&\alpha H^2(S,t)\Sigma+S^{-1}(I-H(S,t))^2\Sigma \nonumber\\
&=\left(\alpha H^2(S,t)+S^{-1}(I-H(S,t))^2\right)\Sigma     \nonumber\\
&<1.5376\cdot\left[\alpha(2/t^2)^2(2/t^2I+S)^{-2}+S(2/t^2I+S)^{-2}\right]\Sigma.  \nonumber
\end{align}
\end{proof}

\subsection{Relative Risks at the Optima}
Note that the Bayes risk (\ref{3.8}), the Bayes prediction risk (\ref{3.13}) and (\ref{3.19}) of ridge are minimized at $\lambda^*=1/\alpha$ ~\cite{Dicker2016Ridge}. In the special case that the distributions of $y\mid\beta_0$ and the prior $\beta_0$ are normal, we know that $\hat{\beta}^{\rm ridge}(\lambda^*)$ is the Bayes estimator, which achieves the optimal Bayes risk (hence certainly the lowest Bayes risk over the whole ridge family). So the Bayes risk of $\hat{\beta}^{\rm mgf}(t)$, for $t \geq 0$, must be at least that of $\hat{\beta}^{\rm ridge}(\lambda^*)$.
Applying the fact that $\lambda=2/t^2$ and $\lambda^*=1/\alpha$, we can set the optima time $t=\sqrt{2\alpha}$ for the MGF.
The following inequality is a key step to obtain
the relative Bayes risk and the Bayes prediction risk of MGF to ridge, when both are optimally tuned.
\begin{lemma}\label{lemma7}
 For all $s_i>0$ and $\alpha>0$, it holds that
\begin{equation}
\alpha H^2\left(s_i,\sqrt{2\alpha}\right)+\frac{\left[1-H(s_i,\sqrt{2\alpha})\right]^2}{s_i} < 1.035 \frac{1}{\alpha(1+s_i)}.  \nonumber
\end{equation}
\end{lemma}
\begin{proof}
The result follows from the numerically computing-\emph{see Supplement}.
\end{proof}
\begin{theorem}\label{theorem2}
Consider the data model (\ref{3.1}), the prior (\ref{3.3}) and the (out-of-sample) feature distribution (\ref{3.9}).\\
It holds that
\begin{equation}\label{bound2}
1\leq \frac{\inf_{t>0}{\rm Risk(\hat{\beta}^{\rm mgf}(t))}}{\inf_{\lambda>0}{\rm Risk}(\hat{\beta}^{\rm ridge}(\lambda))} <1.035.
\end{equation}
Moreover, (\ref{bound2}) also holds if we
replace  the Bayes risk by the Bayes in-sample prediction risk, or the Bayes out-of-sample prediction risk.
\end{theorem}
\begin{proof}
Note that in the special case of a normal-normal likelihood-prior pair,
the minimum of the Bayes risk of $\hat{\beta}^{\rm mgf}(t)$ is not less than 
that of $\hat{\beta}^{\rm ridge}(\lambda^*)$.
But the Bayes risks of MGF (\ref{3.6}) and ridge (\ref{3.8}) do not depend on the likelihood and the prior
(only on their first two moments), thus we  prove the lower bound must be hold in general.
%
%
For the upper bound, set $t=\sqrt{2\alpha}$, and denote the $i$th summand in (\ref{3.6}) and  in (\ref{3.8}) by $a_i$ and $b_i$, respectively.
By Lemma \ref{lemma7}, we have 
\begin{align}
a_i&=\alpha H^2\left(s_i,\sqrt{2\alpha}\right)+\left[1-H(s_i,\sqrt{2\alpha})\right]^2/s_i \nonumber\\
&<1.035\left[1/\alpha(1+s_i)\right]=1.035b_i \nonumber
\end{align}

The proof of the Bayes in-sample prediction risk is similar to 
 Theorem \ref{theorem1} and we omit it here.

Since $H(S, \sqrt{2\alpha})$ and $S$ are diagonal matrices,
the inequality in Lemma \ref{lemma7} can be  extended to matrix operations, i.e.
$\alpha H^2\left(S,\sqrt{2\alpha}\right)+S^{-1}\left(I-H(S,\sqrt{2\alpha})\right)^2 \prec1.035 \alpha(I+S)^{-1}$.
Then for the Bayes out-of-sample prediction risk, we have
\begin{align}
&\alpha H^2(S,\sqrt{2\alpha})\Sigma+S^{-1}(I-H(S,\sqrt{2\alpha}))^2\Sigma \nonumber\\
&=\left[\alpha H^2(S,\sqrt{2\alpha})+S^{-1}(I-H(S,\sqrt{2\alpha}))^2\right]\Sigma  \nonumber\\
&\prec1.035\left[\alpha(I+S)^{-1}\right]\Sigma. \nonumber
\end{align}
\end{proof}

\section{Asymptotic Risk Analysis}\label{section5}

In this section, using random-matrix theoretic techniques (e.g.~\cite{Bai2009SpectralAO}),
we study the Bayes risk of MGF in a high-dimensional asymptotic regime.
Note that the Bayes risk of MGF depend only on the predictor matrix $X$ via the eigenvalues of the
(uncentered) sample covariance $\hat{\Sigma} = X^{\rm T}X/n$.
Thanks to Mar$\check{c}$enko-Pastur (M-P) law~\cite{Marenko1967DISTRIBUTIONOE}, we can explore the limiting spectral distribution of
the sample covariance matrix $\hat{\Sigma}$.
The following assumptions are standard  in random-matrix theory.
Given a symmetric matrix $A\in \mathbb{R}^{p\times p}$, recall its spectral distribution
is defined as $F_{A(x)} = (1/p)\sum_{i=1}^{p}\mathds{1}(\lambda_i(A)\leq x)$, where
$\lambda_i(A), i = 1, \cdots, p$ are the eigenvalues of $A$, and $\mathds{1}(\cdot)$
denotes the $0$-$1$ indicator function.\\
{\bf Assumption A1.} The predictor matrix satisfies $X = Z\Sigma^{1/2}$, for a random matrix $Z\in \mathbb{R}^{n\times p}$
of i.i.d. entries with zero mean and unit variance, and a deterministic positive semidefinite covariance $\Sigma\in \mathbb{R}^{n\times p}$.\\
{\bf Assumption A2.} The sample size $n$ and dimension $p$ both diverge, $n,p\rightarrow\infty$
and converges to a limiting aspect ratio $p/n \rightarrow\gamma\in(0,\infty)$.\\
{\bf Assumption A3.} The spectral measure $F_{\Sigma}$ of the predictor covariance $\Sigma$ converges weakly as
$n,p\rightarrow\infty$ to some limiting spectral measure $G$.

Under the above assumptions, the seminal M-P law can be given immediately.

\begin{theorem}M-P law]\label{theorem3}
Under the assumptions A1–A3, almost surely, the spectral measure $F_{\hat{\Sigma}}$ of $\hat{\Sigma}$
converges weakly to a law $F_{G,\gamma}$, called the empirical spectral distribution, that depends only on
$G,\gamma$.
\end{theorem}
\begin{remark}\label{remark4}
M-P law has a density function, for $\gamma\leq 1$,
\begin{equation}
P_\gamma(s)=\frac{1}{2\pi\gamma\sigma s}\sqrt{(b-s)(s-a)},
\end{equation}
and has a point mass $1-1/\gamma$ at the origin if $\gamma>1$, where $s\in[a,b]$, $a=\sigma^2(1-\sqrt{\gamma})^2$ and
 $b=\sigma^2(1+\sqrt{\gamma})^2$. If $\sigma^2=1$, the M-P law is said to be the standard M-P law.
\end{remark}

The limiting Bayes risk of MGF follows from (\ref{3.6}) directly.
\begin{theorem}\label{theorem4}
Under the assumptions  A1–A3, the data model (\ref{3.1}) and the prior (\ref{3.3}), for $t\geq0$, the Bayes risk (\ref{3.6}) of MGF
converges almost surely to
\begin{equation}\label{5.2}
\sigma^2\gamma\int\left[\alpha_0H^2(s,t)+\frac{(1-H(s,t))^2}{s}\right]dF_{G,\gamma}(s),
\end{equation}
where $\alpha_0=r^2/(\sigma^2\gamma)$, and $F_{G,\gamma}(s)$ is the empirical spectral distribution form M-P law.
\end{theorem}
\begin{proof}
The Bayes risk (\ref{3.6}) of MGF can be rewritten as
$\sigma^2p/n[\int\alpha H^2(s,t)+(1-H(s,t))^2/s]dF_{\hat{\Sigma}}(s)$.
Note that $F_{\hat{\Sigma}}$ converges weakly to $F_{G,\gamma}$ by M-P law, which implies that
$\sigma^2p/n[\int\alpha H^2(s,t)+(1-H(s,t))^2/s]dF_{\hat{\Sigma}}(s)\rightarrow$
$\sigma^2\gamma[\int\alpha_0 H^2(s,t)+(1-H(s,t))^2/s]dF_{G,\gamma}(s)$.
\end{proof}

\begin{remark}
With the same conditions as  Theorem \ref{theorem4},
 for each $\lambda >0$, the Bayes risk (\ref{3.8}) of ridge
converges almost surely to
\begin{equation}\label{5.3}
\sigma^2\gamma\int\frac{\alpha_0\lambda^2+s}{(s+\lambda)^2}dF_{G,\gamma}.
\end{equation}
The proof is similar to Theorem \ref{theorem4} and we omit it here.
The limiting Bayes in-sample prediction risks of MGF and ridge can be found in Supplement 
For the out-of-sample prediction risk, it is difficult to find the explicit limit in general, since (\ref{3.6})
involves both the sample covariance matrix $\hat{\Sigma}$ and the population covariance matrix $\Sigma$,
and (\ref{3.6}) is a complicate function of eigenvalues of $\hat{\Sigma}$.
However, if
the population covariance matrix
$\Sigma$ is an identity matrix,
the limiting Bayes (out-of-sample) prediction risk
follows from Theorem \ref{theorem4} directly.
%
The general case will be considered in our future works.
\end{remark}

\section{Numerical Examples}\label{section6}

In this section, we provide numerical examples to verify the theoretical results of
the relative Bayes risk of MGF to ridge in Section \ref{section4} and its asymptotic risk  expressions in Section \ref{section5},
and
compare the coupling of MGF to ridge and that of GF (gradient flow, which is the continuous-time form of GD) to ridge
under the same settings as~\cite{Ali2019}
(which presents the comparison of the paths of GF and ridge).
%
In details, we generate features via $X=\Sigma^{1/2}Z$, for a matrix $Z$ with i.i.d.
entries from a standard Gaussian and set $\Sigma=I, n=1000, p=500$, and $\sigma^2=r^2=1$.
We set the momentum parameter $D(\mu)={\rm diag}(\mu_i)={\rm diag}(2\sqrt{s_i}+10^{-3})(i=1,\cdots,p)$ in all the experiments,
where $D(\mu)\succ2S^{1/2}$ and $s_i$ is the eigenvalues of the sample covariance matrix.
%
%
And the results for other settings (the results are grossly similar) can be found in \emph{Supplement}.
\begin{figure}[!h]
\centering
\includegraphics[width=9cm, height=5cm]{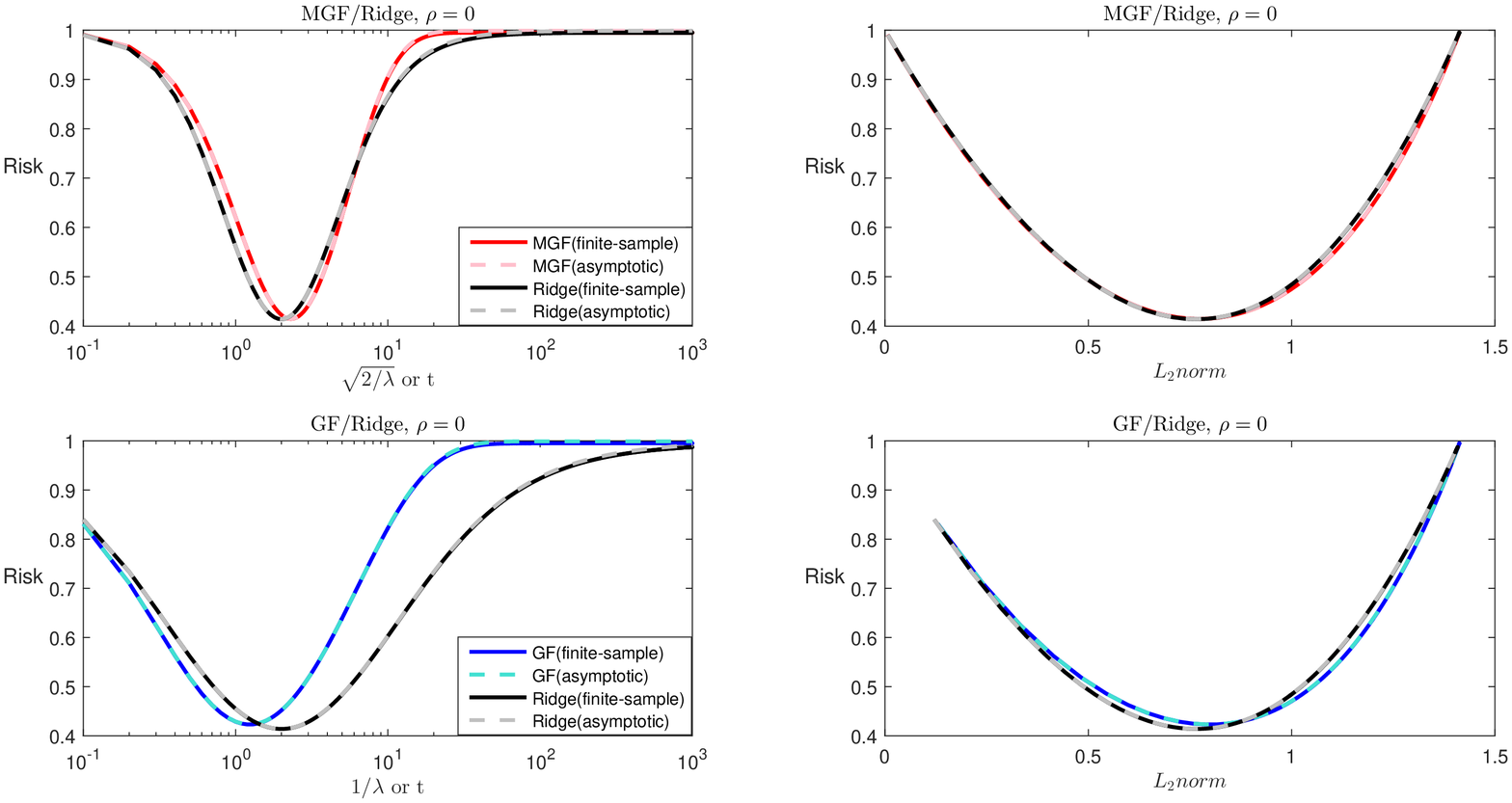}
\caption{\footnotesize{Comparison of the Bayes risks for MGF, GF and ridge, with Gaussian features, $\Sigma=I, n=1000, p=500$.}}\label{figure2}
\end{figure}

In Figure \ref{figure2}, 
we plot MGF versus ridge (calibrated according to $\lambda=2/t^2)$ and GF
 versus ridge (calibrated according to $\lambda=1/t$ ~\cite{Ali2019})  in the first column.
It shows that there is a fairly strong agreement between the risk curves,
and MGF is much closer to ridge than GF over the entire path;
the maximum ratio of the Bayes risk of MGF to ridge is 1.1097 (cf. the upper bound of 1.5376 from Theorem \ref{theorem1})
and the maximum ratio of the optima is 1.0208 (cf. the upper bound of 1.035 from Theorem \ref{theorem2}),
which are lower than that of GF to ridge, which are 1.3663 and 1.0914, respectively
(cf. the  theoretical maximum ratio and the maximum ratio of the optima~\cite{Ali2019} are 1.6862 and 1.2147, respectively);
in addition, it shows that MGF  converges to ridge faster than GD, which is compatible with
the theoretical results (the tuning parameter $\lambda$ of ridge is proportional to
$\mathcal{O}(1/t^2)$ in MGF and  GF requires $\mathcal{O}(1/t)$).
The second column shows the remarkable agreement of the risk over the whole path
when parameterized by the $\ell_2 $ norm of the underlying estimator (more details can be found \emph{Supplement}.
And MGF is closer to ridge than GD, too.
Moreover, the four plots show that the finite-sample and asymptotic risk curves are identical,
which implies that the convergence in Theorem \ref{theorem4} is rapid.

\section{Conclusion}
The present paper studied MGF for the least squares, and characterized the close connections between MGF and ridge.
In theoretical aspect, we proved that
the risk of MGF is no more than 1.54 times that of ridge under the calibration $t=\sqrt{2/\lambda}$.
In particular, the relative Bayes risk of MGF to ridge is between 1 and 1.035 under the optimal tuning.
The numerical experiments showed that the paths of MGF and ridge are strikingly similar.
Compared with GD~\cite{Ali2019}, our work showed the tighter coupling between
MGD and ridge, both theoretically and experimentally.

There are many worthwhile directions for the further work. In particular, it would be interesting to
explore how hyperparameters (e.g. momentum parameters and learning rate) affect the generalization performance
of MGD (or other acceleration optimization  algorithms, e.g. Nesterov) for other linear models.
It would also be interesting to
explain why there is a much tighter coupling of MGF to ridge and that of GF to
ridge under $\ell_2$ norms calibration in theory. More generally, we hope that our work will draw attention to
the exploration of the implicit regularization of the accelerated optimization algorithms in deep learning,
especially deep nonlinear neural networks.

\section*{Acknowledgments}
This work was supported by NSFC-61872076 and Natural Science Foundation of Jilin Province 20200201161JC.
\clearpage
\bibliographystyle{named}
\bibliography{ijcai22}

\begin{thebibliography}{}

\bibitem[\protect\citeauthoryear{Ali \bgroup \em et al.\egroup
  }{2019}]{Ali2019}
Alnur Ali, J.~Zico Kolter, and Ryan~J. Tibshirani.
\newblock A continuous-time view of early stopping for least squares
  regression.
\newblock In {\em Proceedings of the 22th International Conference on
  Artificial Intelligence and Statistics}, volume~89, pages 1370--1378. PMLR,
  2019.

\bibitem[\protect\citeauthoryear{Ali \bgroup \em et al.\egroup
  }{2020}]{Ali2020TheIR}
Alnur Ali, E.~Dobriban, and Ryan~J. Tibshirani.
\newblock The implicit regularization of stochastic gradient flow for least
  squares.
\newblock In {\em The 37th International Conference on Machine Learning}, 2020.

\bibitem[\protect\citeauthoryear{Arora \bgroup \em et al.\egroup
  }{2019}]{Arora2019ImplicitRI}
Sanjeev Arora, Nadav Cohen, Wei Hu, and Yuping Luo.
\newblock Implicit regularization in deep matrix factorization.
\newblock {\em Advances in Neural Information Processing Systems}, 32, 2019.

\bibitem[\protect\citeauthoryear{Bai and Silverstein}{2009}]{Bai2009SpectralAO}
Zhidong Bai and Jack~W. Silverstein.
\newblock Spectral analysis of large dimensional random matrices.
\newblock 2009.

\bibitem[\protect\citeauthoryear{Barrett and
  Dherin}{2021}]{Barrett2021ImplicitGR}
David G.~T. Barrett and Benoit Richard~Umbert Dherin.
\newblock Implicit gradient regularization.
\newblock {\em ArXiv}, abs/2009.11162, 2021.

\bibitem[\protect\citeauthoryear{Dicker and Lee}{2016}]{Dicker2016Ridge}
Dicker and H~Lee.
\newblock Ridge regression and asymptotic minimax estimation over spheres of
  growing dimension.
\newblock {\em Bernoulli}, 22(1):1--37, 2016.

\bibitem[\protect\citeauthoryear{Even \bgroup \em et al.\egroup
  }{2021}]{Even2021ACV}
Mathieu Even, Raphael Berthier, Francis~R. Bach, Nicolas Flammarion, Pierre
  Gaillard, Hadrien Hendrikx, Laurent Massouli'e, and Adrien~B. Taylor.
\newblock A continuized view on nesterov acceleration for stochastic gradient
  descent and randomized gossip.
\newblock {\em ArXiv}, abs/2106.07644, 2021.

\bibitem[\protect\citeauthoryear{Gunasekar \bgroup \em et al.\egroup
  }{2018a}]{Gunasekar2018CharacterizingIB}
Suriya Gunasekar, Jason Lee, Daniel Soudry, and Nathan Srebro.
\newblock Characterizing implicit bias in terms of optimization geometry.
\newblock In {\em Proceedings of the 35th International Conference on Machine
  Learning}, volume~80 of {\em Proceedings of Machine Learning Research}, pages
  1832--1841. PMLR, 2018.

\bibitem[\protect\citeauthoryear{Gunasekar \bgroup \em et al.\egroup
  }{2018b}]{Gunasekar2018ImplicitRI}
Suriya Gunasekar, Blake~E. Woodworth, Srinadh Bhojanapalli, Behnam Neyshabur,
  and Nathan Srebro.
\newblock Implicit regularization in matrix factorization.
\newblock {\em 2018 Information Theory and Applications Workshop (ITA)}, pages
  1--10, 2018.

\bibitem[\protect\citeauthoryear{Li \bgroup \em et al.\egroup
  }{2021}]{Li2021TowardsRT}
Zhiyuan Li, Yuping Luo, and Kaifeng Lyu.
\newblock Towards resolving the implicit bias of gradient descent for matrix
  factorization: Greedy low-rank learning.
\newblock {\em ArXiv}, abs/2012.09839, 2021.

\bibitem[\protect\citeauthoryear{Lin \bgroup \em et al.\egroup
  }{2020}]{Lin2020AcceleratedOF}
Zhouchen Lin, Huan Li, and Cong Fang.
\newblock Accelerated optimization for machine learning: First-order
  algorithms.
\newblock {\em Accelerated Optimization for Machine Learning}, 2020.

\bibitem[\protect\citeauthoryear{Liu \bgroup \em et al.\egroup
  }{2020}]{Liu2020BadGM}
Shengchao Liu, Dimitris Papailiopoulos, and Dimitris Achlioptas.
\newblock Bad global minima exist and sgd can reach them.
\newblock {\em ArXiv}, abs/1906.02613, 2020.

\bibitem[\protect\citeauthoryear{Lyu and Li}{2020}]{Lyu2020GradientDM}
Kaifeng Lyu and Jialun Li.
\newblock Gradient descent maximizes the margin of homogeneous neural networks.
\newblock {\em ArXiv}, abs/1906.05890, 2020.

\bibitem[\protect\citeauthoryear{Mar$\check{\rm c}$enko and
  Pastur}{1967}]{Marenko1967DISTRIBUTIONOE}
V~A Mar$\check{\rm c}$enko and Leonid~A. Pastur.
\newblock Distribution of eigenvalues for some sets of random matrices.
\newblock {\em Mathematics of The Ussr-sbornik}, 1:457--483, 1967.

\bibitem[\protect\citeauthoryear{Morgan and
  Bourlard}{1990}]{1990Generalization}
Nelson Morgan and Herv$\acute{e}$~A Bourlard.
\newblock {\em Generalization and Parameter Estimation in Feedforward Nets:
  Some Experiments}.
\newblock 1990.

\bibitem[\protect\citeauthoryear{Oviedo~Le$\acute{o}$n \bgroup \em et
  al.\egroup }{2021}]{Oviedo2021}
Harry~F. Oviedo~Le$\acute{o}$n, Oscar Dalmau-Cede$\tilde{n}$o, and Rafael
  Herrera.
\newblock An accelerated minimal gradient method with momentum for convex
  quadratic optimization.
\newblock {\em BIT. Numerical mathematics}, 08 2021.

\bibitem[\protect\citeauthoryear{Polyak}{1964}]{POLYAK19641}
B.T. Polyak.
\newblock Some methods of speeding up the convergence of iteration methods.
\newblock {\em USSR Computational Mathematics and Mathematical Physics},
  4(5):1--17, 1964.

\bibitem[\protect\citeauthoryear{Smith \bgroup \em et al.\egroup
  }{2021}]{Smith2021StochasticGD}
Samuel~L. Smith, Benoit Richard~Umbert Dherin, David G.~T. Barrett, and Soham
  De.
\newblock Stochastic gradient descent.
\newblock {\em Machine Learning with Neural Networks}, 2021.

\bibitem[\protect\citeauthoryear{Soudry \bgroup \em et al.\egroup
  }{2018}]{Soudry2018TheIB}
Daniel Soudry, Elad Hoffer, Mor~Shpigel Nacson, Suriya Gunasekar, and Nathan
  Srebro.
\newblock The implicit bias of gradient descent on separable data.
\newblock {\em Journal of Machine Learning Research}, 19(70):1--57, 2018.

\bibitem[\protect\citeauthoryear{Steinerberger}{2021}]{Steinerberger2021OnTR}
Stefan Steinerberger.
\newblock On the regularization effect of stochastic gradient descent applied
  to least-squares.
\newblock {\em ETNA - Electronic Transactions on Numerical Analysis}, 2021.

\bibitem[\protect\citeauthoryear{Suggala \bgroup \em et al.\egroup
  }{2018}]{Suggala2018}
Arun Suggala, Adarsh Prasad, and Pradeep~K Ravikumar.
\newblock Connecting optimization and regularization paths.
\newblock In {\em Advances in Neural Information Processing Systems},
  volume~31. Curran Associates, Inc., 2018.

\bibitem[\protect\citeauthoryear{Tarmoun \bgroup \em et al.\egroup
  }{2021}]{Tarmoun2021UnderstandingTD}
Salma Tarmoun, Guilherme Franca, Benjamin~D Haeffele, and Rene Vidal.
\newblock Understanding the dynamics of gradient flow in overparameterized
  linear models.
\newblock In {\em Proceedings of the 38th International Conference on Machine
  Learning}, volume 139 of {\em Proceedings of Machine Learning Research},
  pages 10153--10161. PMLR, 2021.

\bibitem[\protect\citeauthoryear{Vardi and Shamir}{2021}]{Vardi2021ImplicitRI}
Gal Vardi and Ohad Shamir.
\newblock Implicit regularization in relu networks with the square loss.
\newblock {\em ArXiv}, abs/2012.05156, 2021.

\bibitem[\protect\citeauthoryear{Zhang \bgroup \em et al.\egroup
  }{2016}]{Zhang2016UnderstandingDeepLearning}
Chiyuan Zhang, Samy Bengio, Moritz Hardt, Benjamin Recht, and Oriol Vinyals.
\newblock Understanding deep learning requires rethinking generalization.
\newblock {\em arXiv preprint arXiv:1611.03530}, November 2016.

\bibitem[\protect\citeauthoryear{Zhang \bgroup \em et al.\egroup
  }{2021}]{Zhang2021UnderstandingDL}
Chiyuan Zhang, Samy Bengio, Moritz Hardt, Benjamin Recht, and Oriol Vinyals.
\newblock Understanding deep learning (still) requires rethinking
  generalization.
\newblock {\em Commun. ACM}, 64(3):107--115, 2021.

\bibitem[\protect\citeauthoryear{Zou \bgroup \em et al.\egroup
  }{2021}]{Zou2021TheBO}
Difan Zou, Jingfeng Wu, Vladimir Braverman, Quanquan Gu, Dean~P. Foster, and
  Sham~M. Kakade.
\newblock The benefits of implicit regularization from sgd in least squares
  problems.
\newblock {\em ArXiv}, abs/2108.04552, 2021.

\end{thebibliography}

\end{document}